\documentclass{opt2026arxiv}   % OPT 2026 class (JMLR-based); [anon] removed for the camera-ready
\usepackage[utf8]{inputenc}
\usepackage[T1]{fontenc}
\usepackage{mathtools}
\usepackage{booktabs}
\usepackage[expansion=false]{microtype}
\usepackage{xcolor}
\hypersetup{hidelinks,
  pdftitle={Guarded Block Transfers for Polytope Distance: Linear Convergence with Explicit Constants and Gap-Certified Screening for SVM Training},
  pdfauthor={Mayank Gupta}}
\renewenvironment{proof}[1][]{\par\noindent{\bfseries\upshape Proof#1.}\ }{\jmlrQED}
\newcommand{\conv}{\operatorname{conv}}
\newcommand{\PW}{\mathrm{PW}}
\newcommand{\ip}[2]{\langle #1, #2 \rangle}
\newcommand{\norm}[1]{\lVert #1 \rVert}
\newcommand{\R}{\mathbb{R}}
\newcommand{\gFW}{g^{\mathrm{FW}}}
\newcommand{\gA}{g^{\mathrm{A}}}
\newcommand{\gPW}{g^{\mathrm{PW}}}
\newcommand{\UB}{\mathrm{UB}}
\newcommand{\LB}{\mathrm{LB}}

\title[Guarded Block Transfers for Polytope Distance]{Guarded Block Transfers for Polytope Distance:\\
Linear Convergence with Explicit Constants and\\
Gap-Certified Screening for SVM Training}

\optauthor{%
\Name{Mayank Gupta} \Email{iammayankg@gmail.com}\\
\addr Independent Researcher}

\begin{document}
\maketitle

\begin{abstract}
Hard-margin SVM training is the distance problem between two convex
hulls, which the Triangle Algorithm (TA) solves with a certified
duality gap. We propose \emph{guarded block transfers}: $k$
Mitchell--Demyanov--Malozemov (MDM) transfers per scan are aggregated
into one exact line search, guarded by a comparison with the single
best transfer at a cost independent of the number of points $n$, and replaced by an away
step when the leading transfer is capacity-clipped. On the
Minkowski-difference polytope, where the objective is $1$-strongly
convex, we prove global linear convergence with explicit constants and
no combinatorial swap-step bounds; the per-step contraction is
$k$-independent and only the drop-count bound grows with $k$, so
aggregation is provably safe while the saving in scans is measured, not
proved. Since the TA certificate is the
Frank--Wolfe gap divided by the current distance, gap-safe screening
applies, and we bound the iteration after which it has removed every
zero-weight point off the optimal faces by $O\!\left(\tfrac{k}{\rho}\log\tfrac{1}{\tau}\right)$
up to logarithmic factors in the geometry, $\rho$ the rate constant and
$\tau$ the optimal-face margin. On five LIBSVM benchmarks, with fairer baselines than our submission,
the solver matches the baselines to four or five digits wherever both
converge, ties LIBSVM on sparse-support hard margins and is up to
$1.6\times$ faster than blended pairwise conditional gradients (BPCG). But our full-cache
sequential minimal optimisation (SMO) code ties or beats it on kernel
$L_2$ margins, LIBLINEAR's dual solver is up to $4.5\times$ faster on the
separable $L_2$ margin, and our solver is one to over three orders of magnitude slower than
LIBLINEAR's Newton solver on dense-support $L_2$
margins and than LIBSVM on a nearly touching kernel hull. Its edge is
the certificate, not speed.
\end{abstract}

\section{Introduction}

Training a linear hard-margin SVM is equivalent to finding the closest
points of the two classes' convex hulls: the optimal separator is the
perpendicular bisector of the witness pair and the margin is half the
hull distance \cite{bennett2000duality}. Kalantari's Triangle Algorithm
(TA) exploits this \emph{distance duality} directly \cite{kalantari2015}:
TA~I decides whether the hulls intersect, and TA~II computes their
distance while maintaining an iterate pair $(p,q)$ and a certified
interval $[\LB,\UB]$ from parallel supporting hyperplanes. The enhanced
TA (ETA) of
\cite{gupta2026} adds dot-product caching and search heuristics. Below,
ETA denotes that solver with our guarded block transfers, K-ETA its
kernelised form, and ``our SMO'' our implementation of sequential
minimal optimisation \cite{platt1998}.

The workhorse step of such solvers is the \emph{weight transfer} from
the worst active vertex to the best candidate, the
Mitchell--Demyanov--Malozemov (MDM) step \cite{mdm1974}, which is the
pairwise Frank--Wolfe (FW) step of \cite{lj2015} on the relevant
polytope. Each step costs an $O(n)$ scan and yields one transfer; we ask how much
of the scan can be reused across several transfers without weakening
the guarantee. Our contributions:

\begin{enumerate}\setlength\itemsep{0pt}
\item \textbf{Guarded block transfers} (\S\ref{sec:alg}): pair the top-$k$
receivers with the worst-$k$ active donors, take one exact line search
along the aggregated direction, \emph{guard} it by an $O(k^2)$
comparison with the single MDM step, and fall back to an away step on
capacity-clipped iterations.
\item \textbf{Linear convergence with explicit constants}
(Theorem~\ref{thm:linear}): on $Z=\conv(V)-\conv(W)$ the guarded block
algorithm inherits the pyramidal-width rate of pairwise FW up to a
constant, and the away fallback replaces the combinatorial swap-step
bound by linear drop-step counting (cf.\ \cite{tsuji2022}).
\item \textbf{Screening activation with explicit constants}
(Theorem~\ref{thm:activation}): the TA certificate is the FW gap divided
by the distance (Lemma~\ref{lem:gapid}), so gap-safe screening
\cite{ndiaye2017} applies, and we bound the iteration by which it has
removed every zero-weight point off the optimal faces.
\item \textbf{A regime study on real data} (\S\ref{sec:exp}): two
batteries on five LIBSVM datasets locate the boundary from both sides.
\end{enumerate}
Code, proofs, numerical checks and experiment scripts are at
\url{https://github.com/iammayankg/triangle-algorithm-sv}.

\section{Setting}\label{sec:setting}

Let $V=\{v_1,\dots,v_n\}$, $W=\{w_1,\dots,w_m\}\subset\R^d$. The
polytope distance problem is
\begin{equation}\label{eq:P}
  \min_{p\in\conv(V),\ q\in\conv(W)} \ \tfrac12\norm{p-q}^2 .
\end{equation}
Writing $z=p-q$ and $Z:=\conv(V)-\conv(W)=\conv\{v_i-w_j\}$,
\eqref{eq:P} is the minimum-norm-point problem $\min_{z\in Z}F(z)$,
$F(z)=\tfrac12\norm z^2$, whose objective is $1$-strongly convex and
$1$-smooth. Let $z^*$ be the minimiser, $\delta^*=\norm{z^*}$,
$h(z)=F(z)-F(z^*)$, $M=\operatorname{diam}(Z)$ and
$\delta=\PW(A)>0$ the pyramidal width \cite{lj2015} of the atom set
$A=\{v_i-w_j\}$ (a property of the atom set, not of $Z$ alone).

The algorithm keeps simplex weights $\alpha$ on $V$, $\beta$ on $W$
(iterate $z_t$ after $t$ iterations),
with $p=\sum\alpha_iv_i$, $q=\sum\beta_jw_j$; the induced product
representation $\lambda_{ij}=\alpha_i\beta_j$ of $z$ over the atoms
$v_i-w_j$ has support $\operatorname{supp}\alpha\times\operatorname{supp}\beta$.
With $g=-\nabla F(z)=q-p$, the \emph{side scores} are
$s^V_i=\ip{g}{v_i}$ and $s^W_j=\ip{p-q}{w_j}$, maintained incrementally
from cached arrays.

\paragraph{Certificate.} For any iterate, $\UB=\norm z$ and
$\LB=\min_{u\in Z}\ip{z}{u}/\norm z=\big(\min_i\ip{z}{v_i}-\max_j\ip{z}{w_j}\big)/\UB$
satisfy $\LB\le\delta^*\le\UB$; the TA stops when the certified
relative gap $(\UB-\LB)/\UB$ is at most $\varepsilon$. The $L_2$ soft margin with regularisation
parameter $C$, the $\nu$-SVM and the kernel variants reduce to
\eqref{eq:P} exactly (augmented coordinates that turn the Gram matrix
$K$ into $K+I/C$; reduced convex hulls; feature-space images), see
Appendix~\ref{app:red}.

\paragraph{Gaps.} $\gFW=\max_{u\in Z}\ip{g}{u}-\ip{g}{z}$,
$\gA=\ip{g}{z}-\min_{u\in\operatorname{supp}\lambda}\ip{g}{u}$,
$\gPW=\gFW+\gA$. Because $\ip{g}{v_i-w_j}=s^V_i+s^W_j$ is separable and
the support is a product set, each gap splits into a $V$-side and a
$W$-side term computable from the score arrays in $O(n+m)$
(Lemma~\ref{lem:sep}), written $\gPW_V,\gPW_W$ (likewise for FW and A),
with $g_{\text{side}}$ the term of the side being stepped; a transfer
$v_r-v_u$ on the $V$ side is a pairwise direction of $Z$
(Lemma~\ref{lem:side}).

\section{Guarded block transfers}\label{sec:alg}

\begin{algorithm2e}[t]
\caption{One iteration of the guarded block-transfer step (side $V$; $W$ symmetric)}
\label{alg:block}
\SetAlgoLined\DontPrintSemicolon\LinesNumbered
top-$k$ receivers by $s^V$; worst-$k$ donors in $\operatorname{supp}\alpha$; pair best-with-worst: $(r_j,u_j)$, $j\le k'$ ($k'\le k$ pairs, limited by the support size); pair 1 is the MDM pair $(r^*,u^*)$\;
$d_j\gets v_{r_j}-v_{u_j}$;\ $\mathrm{gap}_j\gets s^V_{r_j}-s^V_{u_j}$;\ discard pairs with $\mathrm{gap}_j\le0$ and receivers that are donors (if $r^*$ is a donor it leaves the donor list, so pair 1 survives);\ $\gamma_j\gets\min(\mathrm{gap}_j/\norm{d_j}^2,\ c_j)$ \tcp*{$c_j=\alpha_{u_j}$: donor capacity}
\eIf{pair 1 uncapped ($\gamma_1=\mathrm{gap}_1/\norm{d_1}^2$) {\normalfont[case (a)]}}{
  $\Delta_B\gets$ gain of exact line search $\theta\in[0,1]$ along $d=\sum_j\gamma_jd_j$;\quad $\Delta_1\gets$ gain of the MDM step\;
  take the step with the larger gain \tcp*{guard, $O(k'^2)$}
}{
  \eIf{the away step on $u^*$ is boundary-clipped {\normalfont[case (b)]}}{
    take it \tcp*{drop step: $\alpha_{u^*}\to0$}
  }{
    take the best of \{block, MDM, away($u^*$), toward($r^*$)\}\;
  }
}
\end{algorithm2e}

Algorithm~\ref{alg:block} amortises one $O(n)$ scan over up to $k$
transfers. Only the scan is amortised: a new receiver still costs one
Gram or kernel column, $O(nd)$. Every
candidate gain is available from cached quantities at a cost
independent of $n$: with $S=\ip{g}{d}=\sum_j\gamma_j\mathrm{gap}_j$ and
$D=\norm d^2=\gamma^\top\Gamma\gamma$, $\Gamma$ the $k'\times k'$ Gram
matrix of the pair directions ($O(k'^2)$ cached dot products), the block
gain is $\theta S-\theta^2D/2$ at $\theta=\min(S/D,1)$; the MDM, away and
toward gains are the usual one-dimensional expressions. The \emph{gain} of a step is
the decrease $h(z_t)-h(z_{t+1})$. The cache update for a block is one
matrix--vector product with the net per-index coefficients; columns are
computed on first use and cached, so support columns cost nothing after
their first appearance. Feasibility is immediate, since donors lose at
most $\theta\gamma_j\le c_j$. We call an iteration \emph{case (a)} if
pair 1 is uncapped and \emph{case (b)} otherwise; the away step of
line 6, which drives $\alpha_{u^*}$ to zero, is a \emph{drop} step
(a capped block pair may also zero a donor, which only helps the
counts below).

\paragraph{Safe screening.} With $r=\sqrt{\UB^2-\max(\LB,0)^2}$ and
$v_{\min}=\arg\min_i\ip{z}{v_i}$, a zero-weight point with
$\ip{z}{v_i}-\ip{z}{v_{\min}}>r\norm{v_i-v_{\min}}$ is outside the
optimal support and is removed (symmetrically on $W$ with the max).
Because the optimum's support survives, the reduced problem has the
same optimum and all subsequent lower bounds remain valid.

\section{Analysis}\label{sec:theory}

Full proofs and the conventions for degenerate cases are in
Appendix~\ref{app:proofs}.

\begin{lemma}[Block gain]\label{lem:block}
Let $b_j=\gamma_j\norm{d_j}$ over the $k'$ retained pairs, all with
$\mathrm{gap}_j>0$. Then $S\ge\sum_jb_j^2$, $D\le k'S$, and
the realised block gain satisfies
$\Delta_B\ge S/(2k')\ge\Delta_1/(2k')$, $\Delta_1$ the MDM gain.
\end{lemma}
With the guard, every case-(a) iteration gains
$\max(\Delta_B,\Delta_1)\ge\Delta_1$: the guarded step never makes less
progress than pairwise FW.

\begin{lemma}[Good-step progress]\label{lem:good}
If the iteration works on the side with the larger pairwise gap, then
on any non-drop iteration $\mathrm{gain}\ge\min\!\big((\gPW)^2/(32M^2),\ \gPW/8\big)$.
\end{lemma}
\begin{theorem}[Global linear convergence]\label{thm:linear}
Suppose the guarded block-transfer algorithm with away fallback is run
for one side per iteration (the side with the larger pairwise gap),
using any $k\ge1$, from any vertex pair (initial support $s_0=2$).
Let $h_t=h(z_t)$. Then every non-drop iteration satisfies
$h_{t+1}\le(1-\rho)h_t$ with
\[
\rho=\tfrac1{16}\,(\delta/M)^2 .
\]
For any horizon $T$, among the first $T$ iterations the number of drop
steps is at most
$k\,T_{\mathrm{good}}+s_0$, where $T_{\mathrm{good}}$ is the number of
non-drop iterations; consequently,
$h_T\le h_0\,(1-\rho)^{(T-s_0)/(k+1)}$. Under the two-sided
schedule of the implementation (larger-gap side first, second side
skipped after a drop) every non-drop iteration has the same
contraction, the drop count is at most $2kT_{\mathrm{good}}+s_0$, and
$h_T\le h_0\,(1-\rho)^{(T-s_0)/(2k+1)}$ (Lemma~\ref{lem:twosided}).
\end{theorem}
The rate combines Lemma~\ref{lem:good} with $\gPW\ge\delta\sqrt{2h}$
and $\gPW\ge\gFW\ge h$ (Lemma~\ref{lem:gsc}). Drop counting is the away-step argument: each
drop removes one support index and non-drop steps add at most $k$.
No swap-step bound is needed: a case-(b) iteration either drops
(counted above) or takes the best of four candidates, of which the
toward or away step alone gains Lemma~\ref{lem:good}'s amount,
whichever wins.

\begin{proposition}[Near-orthogonal speed-up]\label{prop:orth}
If all $k'$ pairs are uncapped and
$|\ip{d_i}{d_j}|\le\eta\norm{d_i}\norm{d_j}$ for $i\ne j$, then the
block gain is at least $(\sum_j\Delta_j)/(1+\eta(k'-1))$, where
$\Delta_j$ is pair $j$'s individual gain.
\end{proposition}
As $\eta\to0$ one scan realises the sum of all $k$ pairwise gains.

\begin{lemma}[The certified gap is the FW gap]\label{lem:gapid}
At any full scan, with $\LB$ evaluated at the current $z\ne0$,
$\UB-\LB=\gFW(z)/\UB$ (a retained best lower bound gives $\le$), and
$\gFW(z)\le2h+R\sqrt{2h}$ with $R=\max_{u\in Z}\norm u$, the radius of
$Z$ about the origin.
\end{lemma}

\begin{theorem}[Screening activation]\label{thm:activation}
Assume $\delta^*>0$. Let $F_V=\{i:\ip{z^*}{v_i}=\min_j\ip{z^*}{v_j}\}$
and, with the max, $F_W$ be the optimal faces: the vertices of zero
score margin, which contain the optimal supports (possibly strictly). Let
$\tau_i=\ip{z^*}{v_i}-\min_j\ip{z^*}{v_j}$ be the score margin of $v_i$
(and symmetrically on $W$), with $\tau>0$ the minimal nonzero margin.
Finally, let $L$ be the larger class diameter.

A zero-weight point with margin $\tau_i$ is screened as soon as
$r<\tau_i/(3L)$. Consequently, once
$h_t<H:=\tau^4\delta^{*2}/(2592\,R^4L^4)$, every zero-weight point of
$V$ outside $F_V$ and of $W$ outside $F_W$ is removed permanently at
the next screening round. Under Theorem~\ref{thm:linear} this holds for
every iteration $t>s_0+\tfrac{a_k}{\rho}\log\max\{1,h_0/H\}$, with
$a_k=k+1$ one-sided and $2k+1$ two-sided.
Following this iteration, each subsequent scan costs
$O(|F_V|+|F_W|+|\operatorname{supp}\alpha|+|\operatorname{supp}\beta|)$
instead of $O(n+m)$.
\end{theorem}
In words: since $r^2\le2R(\UB-\LB)$, a point with margin $\tau_i$ is
screened once the certified gap falls below $\tau_i^2/(18RL^2)$. The
faces can strictly contain the supports: for $V=\{(1,0),(1,1),(2,0)\}$
and $W=\{0\}$, $z^*=(1,0)$ and $(1,1)$ has zero weight and zero margin,
so no gap-safe rule can remove it. 
Lemmas~\ref{lem:block}, \ref{lem:gapid} and~\ref{lem:twosided},
Proposition~\ref{prop:orth} and Theorem~\ref{thm:activation} are
verified numerically on random and adversarial ensembles, at machine
precision or to $10^{-9}$ (Appendix~\ref{app:num}).

\section{Experiments}\label{sec:exp}

\paragraph{Setup.} All solvers are validated against an exact
QP (distances agree to $10^{-8}$ or better; LIBSVM's to $10^{-5}$). Stopping rule
$\varepsilon=10^{-3}$ on the certified relative gap unless stated;
baselines use matched tolerances on the same objective, LIBLINEAR up
to its intercept penalty (Appendix~\ref{app:details}). Synthetic instances follow
\cite{gupta2026} (two Gaussian clouds, $n=5{,}000$ per class). Real data
are the LIBSVM sets a9a, w8a, ijcnn1, gisette and covtype (100k
subsample), standardised, official test sets where they exist, five
seeds (three in the step ablation), 600\,s per solver and cell (LIBLINEAR
\cite{fan2008} uncapped), on two AMD EPYC machines and, for some September runs, an Intel one
(Appendix~\ref{app:details}).

\paragraph{Blocks and screening (Appendix~\ref{app:details}).} In
isolated runs (Table~\ref{tab:kreal}), with the prioritised receiver
search the guard and fallback (guarded $k{=}1$) cut MDM's iterations
$1.4$--$2.7\times$, but an exact search does as much and with it the
guard adds nothing. Aggregation cuts iterations a further
$3.5$--$7.1\times$ over exact-search MDM, which buys wall-clock only where
per-iteration $O(n)$ work, not column computation, dominates: on the
four kernel cells $1.2$--$1.9\times$ over exact-search MDM and
$2.1$--$2.8\times$ over guarded $k{=}1$, nothing on gisette, where
columns take $82$--$89\%$ of the time. The cut over MDM ($6$--$19\times$
on real data, $7$--$19\times$ on synthetic) is smallest where drops are
frequent ($36\%$ of gisette's iterations) or pair directions correlate
($0.26$--$0.29$ on ijcnn1 and covtype, Proposition~\ref{prop:orth}). Screening,
measured with the toward-step solver (Figure~\ref{fig:shrink}), gives
$1.3$--$2.4\times$ at $\varepsilon=10^{-5}$ and $0.9$--$1.3\times$ at
$10^{-3}$.

\begin{table}[t]\centering\scriptsize\setlength{\tabcolsep}{3.5pt}
\caption{Regime battery on real data, five seeds, 600\,s per solver,
mean $\pm$ 95\% CI. Kernel cells (RBF with width parameter $1/d$) use a
class-balanced subsample of $n_k$ points. Same objective per row: LIN,
linear hard margin (ETA with screening vs LIBSVM's \cite{chang2011} $C{=}10^6$
surrogate; our SMO in brackets); KHM, feature-space hard margin (K-ETA
vs the same surrogate); KL2, kernel $L_2$ margin, $C{=}1$ (K-ETA vs
kernel SMO, exact objective; in brackets LIBSVM on the precomputed
$K+I/C$ kernel, fit time $+$ kernel construction, September machine).
conv = seeds in which ETA certified; oracle = $O(nd)$ column (ETA) or
row (SMO) evaluations, thousands; both cache every column or row they
compute; the ETA column is K-ETA on kernel rows. Released code, AMD EPYC
9754 (Appendix~\ref{app:details}).}\label{tab:regime}
\begin{tabular}{llrrlcllr}
\toprule
cell & dataset & $n_k$ / $d$ & ETA (s) & baseline (s) & conv & dist.\ ETA / base & acc.\ ETA / base & oracle \\
\midrule
LIN & gisette & 6000 / 5000 & $38.2\pm3.2$ & $38.5\pm3.1$ [$34.0\pm3.2$] & 5/5 & 3.6638 / 3.6640 & .981 / .981 & 1.5 [1.4] \\
\midrule
KHM & gisette & 6000 / 5000 & $64.6\pm6.2$ & $75.6\pm3.1$ & 5/5 & .05215 / .05215 & .979 / .979 & 2.6 \\
KHM & ijcnn1 & 9853 / 22 & $242\pm69$ & $1.1\pm0.2$ & 5/5 & .00430 / .00430 & .946 / .946 & 2.6 \\
KHM & a9a & 10000 / 123 & $>600$ & $18.5\pm2.0$ & 0/5 & $\le$.00002 / .0033 & .685 / .752 & 6.0 \\
KHM & covtype & 10000 / 54 & $>600$ & $122\pm33$ (maxiter) & 0/5 & $\le$.00044 / .00025 & .711 / .712 & 7.6 \\
KHM & w8a & 6479 / 300 & $7.2\pm20.1$ & --- & 5/5 & hulls intersect & --- & 0.5 \\
\midrule
KL2 & gisette & 6000 / 5000 & $89.8\pm7.8$ & $90.6\pm8.5$ [$0.5\pm0.1+3$] & 5/5 & .06882 / .06882 & .971 / .971 & 3.7 / 3.6 \\
KL2 & w8a & 6479 / 300 & $6.3\pm1.0$ & $5.2\pm1.5$ [$0.6\pm0.1+1$] & 5/5 & .05682 / .05682 & .9745 / .9745 & 3.6 / 3.6 \\
KL2 & a9a & 10000 / 123 & $13.2\pm1.0$ & $8.9\pm1.5$ [$2.2\pm0.2+2$] & 5/5 & .03019 / .03019 & .795 / .795 & 7.6 / 7.6 \\
KL2 & covtype & 10000 / 54 & $16.8\pm4.3$ & $5.8\pm1.2$ [$2.1\pm0.1+2$] & 5/5 & .02609 / .02609 & .792 / .792 & 9.1 / 9.1 \\
KL2 & ijcnn1 & 9853 / 22 & $5.3\pm0.3$ & $2.6\pm0.3$ [$1.1\pm0.1+2$] & 5/5 & .04790 / .04790 & .959 / .959 & 4.6 / 4.5 \\
\bottomrule
\end{tabular}
\end{table}

\paragraph{Real data, regime battery (Table~\ref{tab:regime}).}
TA~I first certifies separability (gisette yes, the other four no, in
$0.1$--$30$\,s). On gisette ($d{=}5{,}000$; $1{,}237$ of $6{,}000$
points in the support) ETA ties LIBSVM on the linear hard margin and our
full-cache SMO has a $1.1\times$ lower mean (overlapping intervals);
K-ETA is $1.2\times$ faster than LIBSVM on the RBF hard margin, with
identical solutions. On the kernel $L_2$ margins K-ETA computes as many
kernel columns as kernel SMO computes rows; SMO ties it on gisette and
w8a and is $1.5$--$2.9\times$ faster elsewhere (our submission's apparent
advantage came mostly from a 512-row SMO cache). Against BPCG \cite{tsuji2022} on the
same supports, cache and certificate (Table~\ref{tab:bpcg}), block ETA
needs $7$--$17\times$ fewer iterations, is $1.2$--$1.6\times$ faster on
four kernel cells and ties or trails slightly on gisette.
LIBSVM on the precomputed $K+I/C$ kernel (brackets, September run) was
$2.7$--$22\times$ faster than K-ETA in that run.

The KHM rows show where the geometry turns against the method. On
ijcnn1 both return the same solution ($819$ support vectors on
average), but
K-ETA needs $1.5$--$2.7\times10^5$ iterations and $242$\,s against
LIBSVM's $1.1$\,s: the hulls are
$0.004$ apart in a space of diameter $\approx\sqrt2$, so a relative
certificate needs an absolute gap of order $10^{-6}$, which plausibly suits LIBSVM's second-order selection better. On a9a and covtype the hulls are within $10^{-4}$ and
$6\times10^{-4}$ of touching, and neither certifies anything (LIBSVM's
surrogate is a soft margin on a9a and hits its iteration cap on covtype). On w8a the hulls intersect (duplicate points, opposite labels);
K-ETA detects this in under $0.1$\,s in four seeds and $36$\,s in
the fifth.

\paragraph{Real data, $L_2$ soft margins (Table~\ref{tab:l2}, appendix).}
The same boundary appears within one problem class. Gisette is
separable, so its $L_2$ margin has a sparse support and ETA converges to
the certified gap in $19$--$21$\,s at every $C$; our full-cache SMO ties
($1.1$--$1.3\times$ lower means, overlapping intervals). LIBLINEAR's
primal Newton solver, which the battery ran, needs $4$\,min to $1.2$\,h
at the matched tolerance $10^{-6}$ and stops $3$--$24\%$ above the
optimum at its default $10^{-4}$; its dual solver reaches the certified
optimum in $2.9$--$4.5$\,s at every tolerance, $2.9$--$4.5\times$ faster
than ETA at $C{=}1$ on the same machine (Tables~\ref{tab:libdef}
and~\ref{tab:kreal}). On gisette ETA adds only the certificate. On the four overlapping datasets LIBLINEAR's Newton solver finishes in
$0.1$--$24$\,s; ETA
converges within the cap only on a9a (four of five seeds) and ijcnn1 at
$C{=}0.1$ ($583$ and $336$\,s, against our SMO's $163$ and $59$\,s), at $C{=}10$ its
capped iterate loses $1$--$27$ accuracy points, and on covtype it is not
usable at any $C$. Uncapped, ETA certifies $10^{-3}$ on a9a, w8a and
ijcnn1 at $C{=}0.1$ in $259$--$792$\,s, $30$--$1{,}800\times$ LIBLINEAR's time across machines
(Table~\ref{tab:t1}), at a rate at least linear but with $O(n^2)$-type
cost (Figure~\ref{fig:trace}).

\section{Related work}
Distance duality and the TA \cite{kalantari2015,gupta2026}; MDM,
Gilbert and Wolfe \cite{mdm1974,lopez2015,gilbert1966,wolfe1976}; FW
for polytope distance and SVMs
\cite{gartner2009,clarkson2010,tsang2005,nanculef2014,keerthi2000,bennett2000duality};
away steps, pyramidal width and facial distance
\cite{guelat1986,lj2015,pena2019}. Boosted FW \cite{combettes2020}, swap-step-free pairwise variants \cite{tsuji2022,rinaldi2023} and block-coordinate FW
\cite{lacoste2013bcfw} are the closest algorithmic relatives; gap-safe
screening \cite{fercoq2015,ndiaye2017,shibagaki2016} and active-set
identification \cite{bomze2020,garber2020} the closest to
Theorem~\ref{thm:activation}, with \cite{elghaoui2012,ogawa2013,wang2014}
as its SVM ancestors (Appendix~\ref{app:related}).

\section{Conclusion}
Guarded block transfers give the Triangle Algorithm a linear rate with
explicit constants and no combinatorial swap-step bounds, and a safe
screening rule with a finite activation bound. With fairer baselines the
experiments narrow the niche: on well-separated hulls the solver ties
LIBSVM and beats BPCG by at most $1.6\times$, but a full-cache SMO and LIBLINEAR's dual solver are as
fast or faster; on dense supports or nearly touching hulls LIBLINEAR's
Newton solver or LIBSVM is one to over three orders of magnitude faster. What
it adds is the certificate and the intersection test, not speed. Limitations: screening pays only when the certificate closes; the code
is NumPy-level and dense; the constants are loose; the $\nu$-SVM
variant is unanalysed; $k$ is fixed (see Appendix~\ref{app:details}).

\paragraph{Acknowledgements.} We thank the OPT 2026 reviewers; one of
them supplied the example after Theorem~\ref{thm:activation}.

\paragraph{AI assistance.} AI tools were used in the preparation of the
code, experiments and manuscript; the author verified all results and
is responsible for the content.

\appendix
\section{Proofs}\label{app:proofs}

\paragraph{Conventions.} $M>0$ (a singleton $Z$ is solved at
initialisation), and the algorithm stops with ``intersect'' when
$\UB=\norm z=0$, before any division by $\UB$. Lemma~\ref{lem:block}
and Algorithm~\ref{alg:block} presuppose $k'\ge1$: a side with no
positive-gap pair is skipped, and on a non-drop iteration the side with
the larger pairwise gap has one. An away direction whose donor has
weight one is unavailable (the iterate is that vertex and the away gap
of that side is zero). In Proposition~\ref{prop:orth}, $\eta\ge0$.

\begin{lemma}[Separability]\label{lem:sep}
With the product representation, $\gPW=\gPW_V+\gPW_W$ where
$\gPW_V=\max_is^V_i-\min_{i\in\operatorname{supp}\alpha}s^V_i\ge0$ and
$\gPW_W$ likewise; the same holds for $\gFW$ and $\gA$.
\end{lemma}
\begin{proof}
$\ip{g}{v_i-w_j}=s^V_i+s^W_j$ is separable and the support is a product
set, so the max over atoms and the min over the support separate. Each
side term is nonnegative: the FW term is the maximum score minus the
weighted mean of the support scores, the away term is that mean minus
the support minimum, and a weighted mean lies between the minimum and
the maximum.
\end{proof}

\begin{lemma}[Side transfers are $Z$-pairwise directions]\label{lem:side}
For $\alpha_u>0$ and any $b\in\operatorname{supp}\beta$,
$v_r-v_u=(v_r-w_b)-(v_u-w_b)$ is a difference of an atom of $Z$ and an
atom in the support, and the transfer $\gamma\le\alpha_u$ realises
$z\leftarrow z+\gamma(v_r-v_u)$ inside $Z$. The side capacity
$\alpha_u=\sum_{b}\alpha_u\beta_b$ is the total weight of the product
atoms $v_u-w_b$, $b\in\operatorname{supp}\beta$: the side transfer is
the sum of the atom-pair transfers of masses $\gamma\beta_b$ from
$v_u-w_b$ to $v_r-w_b$, and for $r\ne u$ this weight update is feasible
exactly when $0\le\gamma\le\alpha_u$; hence every such transfer remains
in $Z$.
\end{lemma}

\begin{proof}[ of Lemma~\ref{lem:block}]
Since $\gamma_j\le\mathrm{gap}_j/\norm{d_j}^2$, we have
$\gamma_j\mathrm{gap}_j\ge\gamma_j^2\norm{d_j}^2=b_j^2$. Consequently,
$S=\sum_j\gamma_j\mathrm{gap}_j\ge\sum_jb_j^2$. By the triangle
inequality, $\norm d\le\sum_j\gamma_j\norm{d_j}=\sum_jb_j$, and by
Cauchy--Schwarz, $(\sum_jb_j)^2\le k'\sum_jb_j^2\le k'S$; hence
$D\le k'S$. Exact line search on the $1$-smooth $F$ gives
$\Delta_B=S^2/(2D)\ge S/(2k')$ in the interior case and
$\Delta_B=S-D/2\ge S/2$ at the boundary $\theta=1$, where $D\le S$. Finally,
$S\ge\gamma_1\mathrm{gap}_1\ge\Delta_1$, since all terms
are nonnegative and
$\Delta_1=\gamma_1\mathrm{gap}_1-\gamma_1^2\norm{d_1}^2/2$.
\end{proof}

\begin{proof}[ of Lemma~\ref{lem:good}]
Let $g_S=\max(\gPW_V,\gPW_W)\ge\gPW/2$. Every direction below is a
difference of two points of one class hull, so its norm is at most
$\operatorname{diam}(\conv V)\le M$. For a feasible direction with
directional derivative $\sigma>0$, squared norm $D$ and maximal step
$\bar\theta$, exact line search on a $1$-smooth function gives
$\mathrm{gain}\ge\min(\sigma^2/(2D),\,\bar\theta\sigma/2)$ (interior
optimum, or the clipped step evaluated with $\bar\theta D\le\sigma$).
In case (a), the guard's gain is at least the MDM gain, whose
derivative is $g_S$ and whose interior optimum is attained, yielding
$\mathrm{gain}\ge g_S^2/(2M^2)\ge(\gPW)^2/(8M^2)$.
In case (b), non-drop, $g_S=\gFW_{\text{side}}+\gA_{\text{side}}$, so
one term is at least $g_S/2\ge\gPW/4$. If it is the FW term, the toward step
(maximal step $1$) gives
$\mathrm{gain}\ge\min((\gFW_{\text{side}})^2/(2M^2),\gFW_{\text{side}}/2)\ge\min((\gPW)^2/(32M^2),\gPW/8)$;
if it is the away term, the away line search is interior (non-drop) and
$\mathrm{gain}\ge(\gA_{\text{side}})^2/(2M^2)\ge(\gPW)^2/(32M^2)$. The
algorithm takes the best available candidate.
\end{proof}

\begin{lemma}[Geometric strong convexity; {\cite[Thm.~6]{lj2015}} with
$\mu=1$]\label{lem:gsc}
For every $z\in Z$ with a representation whose support contains an away
atom, $\gPW(z)\ge\delta\sqrt{2h(z)}$ and $\gPW(z)\ge\gFW(z)\ge h(z)$.
\end{lemma}
\begin{proof}
This is Theorem~6 of \cite{lj2015} for a $1$-strongly convex objective;
we include the short argument. If $z=z^*$ both inequalities are
immediate ($h=0$ and all gaps are nonnegative). Otherwise let
$e=(z^*-z)/\norm{z^*-z}$; the pyramidal-width inequality in the proof
of \cite[Thm.~6]{lj2015} (Definition~3 there, applied to the direction
$g$ together with $\ip{g}{e}\le\norm g$) gives, for the FW atom $s$ and
the away atom $v$, $\ip{g}{s-v}\ge\delta\ip{g}{e}$. By $1$-strong convexity,
$\ip{g}{z^*-z}\ge h+\tfrac12\norm{z^*-z}^2$, so with
$\varrho=\norm{z^*-z}$,
$\gPW\ge\delta(h/\varrho+\varrho/2)\ge\delta\sqrt{2h}$ by AM--GM. Finally
$\gPW=\gFW+\gA\ge\gFW=\max_u\ip{g}{u-z}\ge\ip{g}{z^*-z}\ge h$ by
convexity.
\end{proof}

\begin{proof}[ of Theorem~\ref{thm:linear}]
By Lemmas~\ref{lem:good} and~\ref{lem:gsc} a non-drop iteration gains
at least
\[\min\big(2\delta^2h/(32M^2),\,h/8\big)=(\delta/M)^2h/16\]
using $\delta\le M$; hence $h_{t+1}\le(1-\rho)h_t$. Drop counting:
non-drop
iterations add at most $k$ support indices, drop iterations add none
and remove at least one, and the support size stays $\ge2$ from $s_0$;
so the number of removals, hence of drops, is at most
$kT_{\mathrm{good}}+s_0$, so that
\[T_{\mathrm{good}}\ge(T-s_0)/(k+1).\]
All steps, drops included, are descent steps, so $h$ is monotone.
\end{proof}

\begin{lemma}[Two-sided schedule]\label{lem:twosided}
Let every iteration compute both side gaps, take the guarded step of
Algorithm~\ref{alg:block} on the side with the larger pairwise gap
first, and take the other side's guarded step only if the first step
was not a drop; an iteration counts as a drop iff its first step is a
drop. Then every non-drop iteration gains at least
$\min((\gPW)^2/(32M^2),\,\gPW/8)$ and
the number of drops is at most $2kT_{\mathrm{good}}+s_0$; hence
Theorem~\ref{thm:linear} holds for this schedule with the same $\rho$
and exponent $(T-s_0)/(2k+1)$.
\end{lemma}
\begin{proof}
On a non-drop iteration the first step is exactly the step analysed in
Lemma~\ref{lem:good}, so it gains at least the stated amount; the
second step is an exact line search along a feasible descent direction
(or is skipped), so it cannot increase $h$. Drop counting: a drop
iteration takes a single away step that removes one support index and
adds none, because the second side is skipped; a non-drop iteration
takes two guarded steps and adds at most $k$ indices each. Removals,
hence drops, therefore number at most $2kT_{\mathrm{good}}+s_0$.
\end{proof}
\begin{remark}
The skip is what makes the count work. If the second side always
steps, an iteration whose analysed side drops can add up to $k$ indices
on the other side, and the count becomes
$T_{\mathrm{drop}}\le2kT_{\mathrm{good}}+kT_{\mathrm{drop}}+s_0$, which is
vacuous for $k\ge1$.
\end{remark}

\begin{proof}[ of Proposition~\ref{prop:orth}]
$D=\sum_{i,j}\gamma_i\gamma_j\ip{d_i}{d_j}\le\sum_jb_j^2+\eta\sum_{i\ne j}b_ib_j\le(1+\eta(k'-1))\sum_jb_j^2\le(1+\eta(k'-1))S$
by Cauchy--Schwarz on the cross terms. If $D\ge S$ the line search is
interior and $\Delta_B=S^2/(2D)\ge S/(2(1+\eta(k'-1)))$; if $D<S$ then
$\theta=1$ and $\Delta_B=S-D/2>S/2$, which is larger. For uncapped pairs
$S=\sum_j2\Delta_j$.
\end{proof}

\begin{proof}[ of Lemma~\ref{lem:gapid}]
$\LB=\min_{u\in Z}\ip{z}{u}/\norm z$ by definition of the supporting
hyperplane bound and $\UB=\norm z$; subtracting,
$\UB-\LB=(\ip{z}{z}-\min_u\ip{z}{u})/\norm z=\gFW/\UB$. At the optimum
$0=\ip{z^*}{z^*}-\min_u\ip{z^*}{u}$; subtracting from $\gFW(z)$,
$\gFW(z)=(\norm z^2-\norm{z^*}^2)+(\min_u\ip{z^*}{u}-\min_u\ip{z}{u})\le2h+\max_u\ip{z^*-z}{u}\le2h+R\norm{z-z^*}\le2h+R\sqrt{2h}$
by $1$-strong convexity.
\end{proof}

\begin{proof}[ of Theorem~\ref{thm:activation}]
Strong convexity gives $\norm{z-z^*}\le r$ for the screening radius
$r=\sqrt{\UB^2-\max(\LB,0)^2}$ ($\tau$ is defined only if some point
has positive margin; otherwise there is nothing to screen). Let
$m^*=\min_j\ip{z^*}{v_j}$, choose $j'\in F_V$ and set $v'=v_{j'}$; since $v_{\min}$ minimises $\ip{z}{\cdot}$,
$\ip{z^*}{v_{\min}}-m^*=\ip{z}{v_{\min}-v'}+\ip{z^*-z}{v_{\min}-v'}\le rL$.
Hence $\ip{z}{v_i-v_{\min}}=\ip{z^*}{v_i-v_{\min}}+\ip{z-z^*}{v_i-v_{\min}}\ge\tau_i-rL-r\norm{v_i-v_{\min}}$,
and the rule fires when this exceeds $r\norm{v_i-v_{\min}}$, for which
$\tau_i>3rL$ suffices since $\norm{v_i-v_{\min}}\le L$. For the
activation time, $r^2\le2R(\UB-\LB)$ (for $\LB\ge0$,
$r^2=(\UB-\LB)(\UB+\LB)$ and $\UB+\LB\le2R$; for $\LB<0$, $r^2=\UB^2$
and $\UB\le2(\UB-\LB)$), and
$2R(\UB-\LB)\le2R\gFW/\UB\le(2R/\delta^*)(2h+R\sqrt{2h})\le(4R^2/\delta^*)\sqrt{2h}$
(Lemma~\ref{lem:gapid}, with equality in the first step for the freshly
evaluated bound; the last step uses $h\le R^2/2$, which holds for every
feasible iterate since $h\le\tfrac12\norm z^2$); requiring
$r^2\le\tau^2/(9L^2)$ yields $h_t<H$, and Theorem~\ref{thm:linear} with
$\log(1/(1-\rho))\ge\rho$ converts it into the iteration bound (the
$\max\{1,\cdot\}$ covers $h_0\le H$, where the condition holds from the
start by monotonicity). The removal is executed at the first screening
round after the condition holds; permanence follows from the safety of
the rule.
\end{proof}

\section{Exact reductions}\label{app:red}

\paragraph{$L_2$ soft margin.} The problem
$\min\tfrac12\norm w^2+\tfrac C2\sum_i\xi_i^2$ is the hard-margin problem
on the augmented points $\tilde x_i=[x_i;\,e_i/\sqrt C]$, $e_i$ the
$i$-th standard basis vector, whose Gram
matrix is $K+I/C$ ($K$ the Gram matrix of the original points) and whose
augmented hulls never intersect; let $\delta_C$ denote their distance. The
augmentation is never materialised: the extra coordinates of an iterate
are exactly its convex weights, so every augmented inner product is a
cached original-space product plus a sparse diagonal correction,
$\ip{\tilde p}{\tilde v_i}=\ip{p}{v_i}+\alpha_i/C$ and
$\norm{\tilde p}^2=\norm p^2+\norm\alpha^2/C$. The separator is recovered
as $w=2(p-q)/\delta_C^2$, $b=(\norm{\tilde q}^2-\norm{\tilde p}^2)/\delta_C^2$,
and the primal value is $P=2/\delta_C^2$. Validation: the solver agrees
with an exact QP on the explicitly augmented problem to $10^{-10}$
across random instances and values of $C$; the regularisation path
$\delta_C$ is monotone in $C$; on a9a ($C{=}0.1$) its certified primal
$687.07$ equals LIBLINEAR's to five digits.

\paragraph{$\nu$-SVM.} The $L_1$ (hinge) soft margin is the nearest-point
problem between \emph{reduced} convex hulls
$R(V,\mu)=\{\sum a_iv_i:\ \sum a_i=1,\ 0\le a_i\le\mu\}$,
$\mu=2/(\nu\ell)$ with $\ell=n+m$ the number of training points
\cite{bennett2000duality}. Two changes: the
extreme-point oracle becomes a capped top-$k$ blend (weight $\mu$ on the
$\lfloor1/\mu\rfloor$ best points in the search direction, via a partial
sort, still $O(n)$), and the primary step is the capped MDM transfer with
blended toward steps as fallback; caching and the reduced
support-function lower bound carry over unchanged. Validation: agreement
with the exact box-constrained QP to $10^{-9}$--$10^{-14}$ and with
LIBSVM's NuSVC to $\sim10^{-10}$ through the mapping
$\delta^*=\norm{w_{\mathrm{svc}}}/\varsigma$, $w_{\mathrm{svc}}$ LIBSVM's
weight vector and $\varsigma$ the per-class
sum of the rescaled dual coefficients. The reduced hull is a polytope, so
\eqref{eq:P} still applies, but its vertices are blends and the capped
transfer is not the pairwise step analysed in \S\ref{sec:theory}; the
analysis of this variant is left open.

\paragraph{Kernels.} The algorithm touches data only through inner
products, so kernelisation changes plumbing, not algorithm: Gram columns
become kernel columns $\kappa(V,x_i)$, squared norms become kernel
diagonals,
iterates exist purely as convex weights (the exact cache refresh is
reconstructed from the support's cached kernel columns), and the
classifier is the kernel expansion of the witness-pair bisector. A ridge
$K+I/C$ gives the kernel $L_2$ margin by the same diagonal trick.
Validation: the linear kernel reproduces the Euclidean solver to
$10^{-12}$; RBF distances match the exact kernel QP to $10^{-9}$ and
LIBSVM's hard-margin solution, via
$\norm{w_{\mathcal H}}^2=(\alpha\circ y)^\top K(\alpha\circ y)$, $w_{\mathcal H}$ the
feature-space normal, for
LIBSVM's dual coefficients $\alpha$ and labels $y$, to $10^{-5}$;
kernel SMO matches the QP to $10^{-8}$.

\section{Experimental details}\label{app:details}

\paragraph{Solver settings and tolerance equivalence.} ETA: block
transfers with $k=16$, guarded, away fallback, pairwise anti-zig-zag,
stop at certified relative gap $10^{-3}$, iteration cap $2\times10^6$,
600\,s cap. Every result comes from the released code (schedule of
Lemma~\ref{lem:twosided}, case-(b) fallback in every path). Most tables and figures of real-data and synthetic timings were
rerun in October, after this paper's acceptance, with three
baseline-fairness changes: our SMO keeps every kernel row it computes
(in the submission it had a 512-row cache), LIBLINEAR's dual solver is
run beside its Newton solver on gisette and a9a, and the step ablation
adds exact-search and away-step rows and the w8a, a9a and covtype kernel
cells. Between the two versions K-ETA's kernel $L_2$ times changed by
under $1.5\times$ and kernel SMO's by $2$--$5.5\times$, so the
submission's K-ETA wins over SMO came mostly from that cache. Not rerun:
Tables~\ref{tab:sched} and~\ref{tab:kabl}, the instrumented runs behind
the pair correlations and the cost model of this appendix, and the
September entries named in the hardware paragraph. The submitted
version's single-surviving-pair path took the possibly
capacity-clipped MDM transfer directly, bypassing the fallback; the fix
changes a block trajectory only on such iterations. Block assembly copies the $k'$ cached columns into a $k'\times n$
array, an $O(k'n)$ copy, before the $O(k'^2)$ products. Screening runs
every 25 iterations, only in the LIN cell, where it never activated at $10^{-3}$ (Theorem~\ref{thm:activation} gives a
sufficient condition only, which these runs stop before meeting). Our
SMO \cite{platt1998}: maximal-violating-pair selection, KKT gap $10^{-3}$,
same caps; hard margin as $C=10^{12}$; kernel SMO on $K+I/C$; a row
cache that keeps every row it computes, as ETA's column cache does
(covtype in Table~\ref{tab:l2} still uses the 512-row cache). LIBSVM
on the precomputed $K+I/C$ kernel (scikit-learn \texttt{SVC},
\texttt{kernel='precomputed'}, $C{=}10^6$, tolerance $10^{-3}$, five
seeds; \texttt{src/libsvm\_kl2.py}) solves the same kernel $L_2$
margin: the RBF kernel with $1/C$ on the diagonal is built with BLAS in
$1$--$3$\,s and the fit takes $0.5$--$2.2$\,s, with distances equal to
K-ETA's to five digits. LIBLINEAR \cite{fan2008} (scikit-learn
\cite{pedregosa2011} 1.8 \texttt{LinearSVC} with its default
\texttt{dual=\textquotesingle auto\textquotesingle}, which selects the primal trust-region Newton solver
on these datasets, not dual coordinate descent): squared hinge with parameter
$C/2$, so that its objective equals $\tfrac12\norm w^2+\tfrac C2\sum\xi_i^2$;
tolerance $10^{-6}$, chosen so that its final primal matches ETA's
certified value on every cell where ETA converges except gisette at
$C{=}10$ (a9a $687.07$ vs $687.07$; gisette $0.14900$ vs $0.14893$,
inside ETA's certified $10^{-3}$; ijcnn1 $550.61$ vs $550.61$); this
tolerance is why the Newton solver needs $17{,}550$ iterations and
$32$\,min on gisette at $C{=}1$, and $44{,}569$ iterations and $1.2$\,h
at $C{=}10$, where it nevertheless stops $4.3\%$ above the certified optimum ($0.15532$ vs $0.14899$). Table~\ref{tab:libdef} sweeps the
tolerance for both solvers: at the default $10^{-4}$ the Newton solver
takes $3$--$40$\,s on gisette but stops $3$--$24\%$ above the optimum,
while the dual solver reaches it in $3$--$4.5$\,s at every tolerance; on
a9a the order reverses.
Intercept scaling $100$, so LIBLINEAR regularises $b/100$ where the
geometric formulation leaves $b$ free (the primals agree to five
digits); iteration cap
$2\times10^5$. LIBSVM \cite{chang2011} (scikit-learn \texttt{SVC}): tolerance $10^{-3}$, 1000\,MB cache,
iteration cap $2\times10^6$ (reported as \emph{maxiter} when hit); hard
margins as $C=10^6$; RBF kernel $\exp(-\norm{x-x'}^2/d)$. Oracle counts: ETA counts
Gram or kernel columns evaluated against a full class ($O(nd)$ or
$O(n_kd)$ each); SMO counts kernel rows, the same unit. Columns are
cached without eviction and the block Gram matrix $\Gamma$ is
assembled from four cached entries per pair, so peak cache memory is
the column count times $n_k$ doubles: $0.73$\,GB for the largest cell
(covtype KL2, $9{,}100$ columns), $0.61$\,GB for a9a KL2 and
$0.36$\,GB for ijcnn1 KL2; measured peak resident memory in the
isolated runs of Table~\ref{tab:kreal} is $0.74$\,GB on gisette
(data-dominated), and on the kernel cells $0.49$--$0.93$\,GB for ETA and
$0.67$--$1.27$\,GB for BPCG (per-cell means).

\paragraph{Data and seeds.} a9a, w8a, ijcnn1 and gisette use the LIBSVM
training and test files; covtype (binary, scaled) is split 75/25 by
seed and its training part subsampled to $100{,}000$ points. Features are standardised
on the training set. Kernel cells use a class-balanced subsample of at
most $5{,}000$ points per class (w8a's minority class has $1{,}479$
points, hence $n_k=6{,}479$; gisette uses all $6{,}000$). Seeds vary
the covtype split, the kernel subsamples and the timing environment; on
datasets with fixed splits the deterministic solvers return identical
solutions across seeds, which is why accuracies are reported without
intervals.
Confidence intervals are $t$-based 95\% intervals over seeds.

\begin{figure}[h]\centering
\includegraphics[width=.8\linewidth]{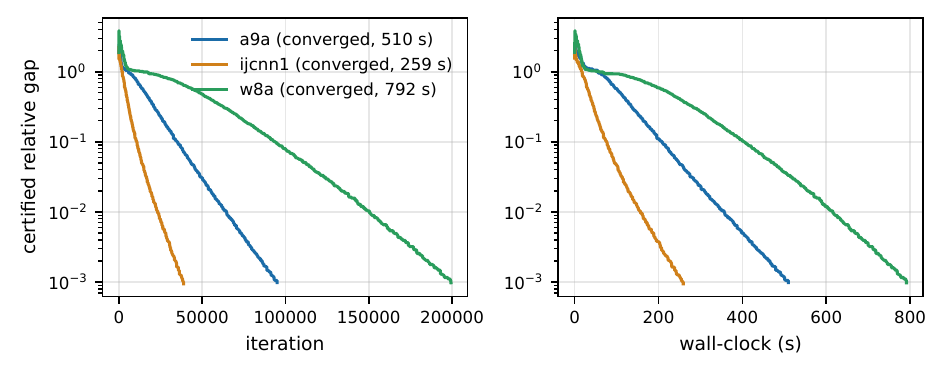}
\caption{Certified relative gap of ETA on a9a, w8a and ijcnn1 ($L_2$, $C=0.1$, uncapped, seed 0, AMD EPYC 7B13), versus iteration
and wall-clock: linear-rate
convergence in the dense-support regime, reaching the certified
$10^{-3}$ in $259$--$792$\,s against the $600$\,s cap of
Table~\ref{tab:l2}. Heuristically ($\delta$ has no known bound in terms
of the class hulls), $n_s\approx0.6n$ active atoms forming a
near-regular simplex have pyramidal width $\sqrt{2/n_s}$ times the side
\cite{lj2015}, so $\rho\sim1/n_s$ and the cost is $O(n^2)$-type.}\label{fig:trace}
\end{figure}

\paragraph{Hardware and concurrency.} The October rerun used two
dedicated cloud VMs with the same software (Linux, Python 3.14.7, NumPy
2.4.4 with OpenBLAS 0.3.31, SciPy 1.17.1, scikit-learn 1.8.0,
\texttt{OMP\_NUM\_THREADS=1}, no other load). A 16-vCPU AMD EPYC 9754
(60\,GB RAM) ran the regime battery (Table~\ref{tab:regime}) and the
synthetic benchmark (Figure~\ref{fig:final}) four single-threaded cells
at a time, the $L_2$ battery (Table~\ref{tab:l2}) up to four at a time,
and the runs of Figure~\ref{fig:shrink} one at a time. A 32-vCPU AMD
EPYC 7B13 (128\,GB RAM) ran the step ablation and the BPCG comparison
(Tables~\ref{tab:kreal} and~\ref{tab:bpcg}) one configuration at a time,
the LIBLINEAR sweep (Table~\ref{tab:libdef}) three processes at a time,
the uncapped traces (Figure~\ref{fig:trace}) one at a time, and
covtype's LIBLINEAR entries. Comparisons are made within one machine
except the trace-to-LIBLINEAR ratios of Section~\ref{sec:exp}, some
entries of Table~\ref{tab:t1} and the brackets of Table~\ref{tab:regime}
(September LIBSVM times beside October K-ETA times); ratios across
machines and loads can differ by more than $1.5\times$ (ETA on the gisette $L_2$ margin at $C{=}1$: $13.0$\,s
in isolation on the EPYC 7B13, $20.6$\,s in the EPYC 9754 battery).
September results, on a 32-vCPU Intel Xeon 2.6\,GHz VM (122\,GB RAM,
Python 3.12): covtype's ETA and SMO entries in Table~\ref{tab:l2}
(four cells at a time, SMO with a 512-row cache; on covtype's
$100{,}000$ training points ETA's and SMO's caches each reach up to
$80$\,GB, and full-cache runs were killed for lack of memory on both
October machines), where both timed out on every
cell; the LIBSVM-precomputed brackets of Table~\ref{tab:regime} (up to
19 concurrent processes); the instrumented pair-correlation and cost-model runs of
this appendix; and Tables~\ref{tab:sched} and~\ref{tab:kabl}, whose
absolute times carry an uncertain load, so only their within-table
ratios should be read. Concurrent load can manufacture ratios: under
September's 19-process load guarded $k{=}1$ appeared $1.3\times$ faster
than MDM on gisette, a tie in isolation; only ratios whose intervals
separate clearly should be read as results.

\begin{table}[h]\centering\scriptsize\setlength{\tabcolsep}{3pt}
\caption{$L_2$ soft-margin battery: five seeds, 600\,s per solver for
ETA and our SMO (our own implementation; LIBSVM would need an $n\times n$
precomputed kernel), LIBLINEAR
(primal Newton solver, Appendix~\ref{app:details}) at tolerance
$10^{-6}$ with no time cap; released code with a full SMO row cache, up to four
single-threaded cells at a time on the AMD EPYC 9754 machine, except
covtype (below). Mean $\pm$ 95\% CI, which on the fixed-split datasets
(all but covtype) measures timing noise only (w8a, $C{=}0.1$: one
LIBLINEAR seed took $71$\,s, the other four $12$--$14$\,s). ``cap'' = stopped at the time or iteration cap in every seed; $^\ast$ four of five seeds converged within
the cap, the mean counting the capped seed at $600$\,s. ETA's certified relative
gap is reported at its stopping point (mean [min, max]); accuracies are
those of the returned iterate. covtype: ETA and SMO from the September
run on the earlier machine with SMO's 512-row cache, because a full
cache needs about $80$\,GB (Appendix~\ref{app:details}); LIBLINEAR on an
AMD EPYC 7B13, whose objective values match the September run's to seven or more
significant digits.}\label{tab:l2}
\begin{tabular}{llrrrll}
\toprule
dataset ($n+m$ / $d$) & $C$ & ETA (s) & SMO (s) & LIBLINEAR (s) & ETA gap at stop & acc.\ ETA / SMO / LIBLINEAR \\
\midrule
gisette (6,000 / 5,000) & 0.1 & $19.3\pm2.4$ & $16.7\pm2.6$ & $261\pm4$ & $9\times10^{-4}$ & .981 / .981 / .981 \\
 & 1 & $20.6\pm3.5$ & $15.9\pm3.3$ & $1{,}893\pm41$ & $10\times10^{-4}$ & .981 / .981 / .981 \\
 & 10 & $20.9\pm1.4$ & $19.6\pm1.1$ & $4{,}380\pm277$ & $10\times10^{-4}$ & .981 / .981 / .978 \\
\midrule
a9a (32,561 / 123) & 0.1 & $583\pm17^{\ast}$ & $163\pm4$ & $1.9\pm0.1$ & $(9.9$--$11)\times10^{-4}$ & .849 / .849 / .849 \\
 & 1 & cap & cap & $3.2\pm0.8$ & 0.66 [0.55, 0.85] & .840 / .850 / .849 \\
 & 10 & cap & cap & $1.1\pm0.1$ & 1.2 [1.2, 1.2] & .594 / .786 / .849 \\
ijcnn1 (49,990 / 22) & 0.1 & $336\pm61$ & $59.0\pm7.4$ & $0.1\pm0.0$ & $9\times10^{-4}$ & .918 / .918 / .918 \\
 & 1 & cap & $334\pm11$ & $0.2\pm0.1$ & 0.17 [0.16, 0.18] & .920 / .918 / .918 \\
 & 10 & cap & cap & $0.2\pm0.0$ & 1.1 [0.99, 1.2] & .649 / .917 / .918 \\
w8a (49,749 / 300) & 0.1 & cap & $470\pm50$ & $24.4\pm32.2$ & 0.044 [0.016, 0.12] & .987 / .987 / .987 \\
 & 1 & cap & cap & $6.8\pm1.6$ & 0.86 [0.84, 0.89] & .979 / .987 / .987 \\
 & 10 & cap & cap & $8.5\pm1.5$ & 1.4 [1.4, 1.5] & .975 / .977 / .987 \\
covtype (100,000 / 54) & 0.1 & cap & cap & $12.8\pm11.1$ & 1.0 [1.0, 1.0] & .572 / .701 / .756 \\
 & 1 & cap & cap & $3.2\pm1.5$ & 1.1 [1.0, 1.1] & .537 / .545 / .756 \\
 & 10 & cap & cap & $2.5\pm0.9$ & 1.3 [1.3, 1.6] & .502 / .515 / .756 \\
\bottomrule
\end{tabular}
\end{table}

\paragraph{Reference solver in the KL2 cells.} The hinge-loss
\texttt{SVC(rbf, C=1)}, a different objective, took $4.5\pm0.4$ (a9a),
$3.1\pm0.3$ (covtype), $75.1\pm4.8$ (gisette), $0.9\pm0.1$ (ijcnn1) and
$2.2\pm0.7$\,s (w8a), with test accuracies $0.789$, $0.783$, $0.977$,
$0.950$ and $0.975$, within one point of the $L_2$ solutions and on
either side of them.

\paragraph{Step ablation, blocks and shrinking (synthetic).} Iteration
counts at $\varepsilon=10^{-6}$ (single instances): $d{=}30$, $n{=}200$:
toward $158{,}941$, midpoint $148{,}996$, away $106$, pairwise $181$;
$d{=}100$, $n{=}1{,}000$: toward and midpoint fail within $300{,}000$,
away $15{,}886$, pairwise $3{,}921$. At the working tolerance $10^{-3}$
the pairwise/away gain over the midpoint heuristic is only $\sim\!11\%$,
which is why the heuristic sufficed in \cite{gupta2026}. Block transfers cut iterations $3$--$25\times$ relative to single MDM
across $k=4$--$32$ ($7$--$19\times$ at $k{=}16$; e.g.\
$43{,}486\to3{,}001$ at $k{=}16$ on a dense-support $L_2$ instance at
$d{=}1{,}000$)
and are the only configuration that reaches the certified tolerance on
MNIST odd-vs-even. Figure~\ref{fig:shrink} shows the screening
speed-ups and Figure~\ref{fig:final} the consolidated comparison,
five trials per class (mean $\pm$ 95\% CI): on the four hard-margin
classes the optimised solver ranges from $1.5\times$ faster than LIBSVM
to $1.1\times$ slower ($0.47\pm0.06$ vs $0.69\pm0.15$\,s at
$d{=}1{,}000$, the only separated intervals; $14.1\pm1.5$ vs
$13.2\pm1.9$\,s at $d{=}10{,}000$; $0.59\pm0.11$ vs $0.63\pm0.15$\,s at
$\varepsilon{=}10^{-5}$, where the paper-style configuration fails to
converge in every trial; RBF $1.64\pm0.42$ vs $1.94\pm0.34$\,s); on the
synthetic $L_2$ margin ($d{=}1{,}000$, $C{=}1$) it ties LIBLINEAR
($13.7\pm11.0$ vs $14.5\pm5.0$\,s); on MNIST odd-vs-even it is
$2.1\times$ slower ($135\pm17$ vs $63\pm13$\,s; single MDM never
converges) and on the $\nu$-SVM $1.4\times$ slower ($41.8\pm7.9$ vs
NuSVC's $30.9\pm5.7$\,s). Distances agree within tolerance on every
row.

\begin{table}[h]\centering\scriptsize
\caption{Two-sided schedules on the synthetic instances of this appendix (September run) ($n{=}5{,}000$ per class, $k{=}16$): the unconditional $V$-then-$W$ order of the submitted version versus the schedule of Lemma~\ref{lem:twosided} (larger-gap side first, second side skipped after a drop). Drop steps and skipped second-side steps are counted over the run.}\label{tab:sched}
\begin{tabular}{lllrrrr}\toprule
instance & schedule & time (s) & iterations & drops & skips & distance \\\midrule
hard margin, $d{=}1000$, $\varepsilon{=}10^{-5}$ & unconditional & 0.8 & 16 & 0 & 0 & 55.339365 \\
hard margin, $d{=}1000$, $\varepsilon{=}10^{-5}$ & Lemma~\ref{lem:twosided} & 0.7 & 16 & 0 & 0 & 55.339365 \\
\midrule
$L_2$ overlap, $d{=}100$, $C{=}1$ & unconditional & 23.1 & 12,311 & 785 & 0 & 0.078400 \\
$L_2$ overlap, $d{=}100$, $C{=}1$ & Lemma~\ref{lem:twosided} & 21.5 & 10,701 & 599 & 316 & 0.078400 \\
\midrule
$L_2$ overlap, $d{=}1000$, $C{=}1$ & unconditional & 14.2 & 2,901 & 182 & 0 & 0.578143 \\
$L_2$ overlap, $d{=}1000$, $C{=}1$ & Lemma~\ref{lem:twosided} & 14.9 & 3,001 & 180 & 121 & 0.578143 \\
\bottomrule
\end{tabular}\end{table}

\begin{table}[h]\centering\scriptsize
\caption{LIBLINEAR's two solvers across tolerances on gisette and
a9a (\texttt{src/liblin\_sweep.py}; five seeds, three processes at a
time and nothing else running, AMD EPYC 7B13). Newton: \texttt{LinearSVC}
with \texttt{dual=\textquotesingle auto\textquotesingle}, which selects
the primal trust-region Newton solver, the one the batteries ran; dual:
\texttt{dual=True}, dual coordinate descent. Primal objective relative
to the best known value $P^*$ ($0.14835$ and $0.14893$ on gisette at
$C{=}0.1$ and $1$, ETA's certified values; $687.07$ and $6{,}869.7$ on
a9a); on these fixed splits every seed returns the same primal and
accuracy. On a9a at $C{=}1$ the dual solver hit its $2\times10^5$
iteration cap in four of five seeds at $10^{-6}$ and in all five at
$10^{-8}$, so those times are lower bounds.}\label{tab:libdef}
\begin{tabular}{lrrrlrlc}\toprule
dataset & $C$ & tol & Newton (s) & Newton primal & dual (s) & dual primal & acc.\ Newton / dual \\\midrule
gisette & 0.1 & $10^{-4}$ & $40.1\pm0.5$ & 0.15345 ($+3.4\%$) & $3.0\pm0.1$ & 0.14835 ($+0.0\%$) & .978 / .981 \\
 & & $3\times10^{-5}$ & $85.7\pm0.8$ & 0.14922 ($+0.6\%$) & $3.3\pm0.2$ & 0.14835 ($+0.0\%$) & .980 / .981 \\
 & & $10^{-5}$ & $151\pm2$ & 0.14842 ($+0.05\%$) & $3.7\pm0.6$ & 0.14835 ($+0.0\%$) & .981 / .981 \\
 & & $10^{-6}$ & $269\pm2$ & 0.14835 ($+0.0\%$) & $4.5\pm0.7$ & 0.14835 ($+0.0\%$) & .981 / .981 \\
gisette & 1 & $10^{-4}$ & $3.0\pm0.2$ & 0.18440 ($+24\%$) & $2.9\pm0.2$ & 0.14893 ($+0.0\%$) & .974 / .981 \\
 & & $3\times10^{-5}$ & $4.0\pm0.3$ & 0.18343 ($+23\%$) & $3.2\pm0.1$ & 0.14893 ($+0.0\%$) & .974 / .981 \\
 & & $10^{-5}$ & $398\pm18$ & 0.15568 ($+4.5\%$) & $3.5\pm0.1$ & 0.14893 ($+0.0\%$) & .977 / .981 \\
 & & $10^{-6}$ & $1{,}813\pm67$ & 0.14900 ($+0.05\%$) & $4.5\pm0.7$ & 0.14893 ($+0.0\%$) & .981 / .981 \\
\midrule
a9a & 0.1 & $10^{-4}$ & $0.5\pm0.1$ & 687.08 ($+0.001\%$) & $131\pm38$ & 687.07 ($+0.0\%$) & .849 / .849 \\
 & & $10^{-6}$ & $1.8\pm0.4$ & 687.07 ($+0.0\%$) & $180\pm44$ & 687.07 ($+0.0\%$) & .849 / .849 \\
 & & $10^{-8}$ & $2.4\pm0.5$ & 687.07 ($+0.0\%$) & $226\pm50$ & 687.07 ($+0.0\%$) & .849 / .849 \\
a9a & 1 & $10^{-4}$ & $0.6\pm0.1$ & 6{,}869.7 ($+0.001\%$) & $1{,}291\pm416$ & 6{,}869.7 ($+0.0\%$) & .850 / .849 \\
 & & $10^{-6}$ & $2.5\pm0.3$ & 6{,}869.7 ($+0.0\%$) & $1{,}828\pm271$ & 6{,}869.7 ($+0.0\%$) & .849 / .849 \\
 & & $10^{-8}$ & $5.9\pm0.8$ & 6{,}869.7 ($+0.0\%$) & $1{,}852\pm249$ & 6{,}869.7 ($+0.0\%$) & .849 / .849 \\
\bottomrule
\end{tabular}\end{table}

\begin{table}[h]\centering\scriptsize
\caption{Time to $1\%$ of the optimum, seed-0 traces. The ETA columns
measure \emph{dual} progress and LIBLINEAR's column \emph{primal}
progress, so this is not a matched primal comparison (the primal of
ETA's returned $(w,b)$ is not recorded in the traces). For ETA the
certified dual value $2/\UB_t^2$ increases to the optimum,
so $s_t=1-(2/\UB_t^2)/P^*$ with $P^*$ the best known primal (the lower
of LIBLINEAR's best value and ETA's certified value); ``certified
$1\%$'' is the first time the certificate itself proves
$1-\LB_t^2/\UB_t^2\le1\%$, and ``cert.\ $10^{-3}$'' the time the stopping
rule $(\UB-\LB)/\UB\le10^{-3}$ is met. ETA's a9a, w8a and ijcnn1 times are from
the uncapped runs of Figure~\ref{fig:trace} (\texttt{src/time\_to\_tol.py});
its gisette times are from the battery traces of Table~\ref{tab:l2}.
Machines: ETA's gisette entries and LIBLINEAR's upper bounds come from
the EPYC 9754, ETA's a9a, w8a and ijcnn1 entries and LIBLINEAR's swept
times from the EPYC 7B13.
LIBLINEAR: the mean time of the coarsest tolerance of Table~\ref{tab:libdef}
whose primal is within $1\%$ of $P^*$ in every seed, for the Newton and
the dual solver; where the sweep did not run, the matched $10^{-6}$
Newton run of Table~\ref{tab:l2}, an upper bound (on gisette at
$C{=}10$ that run stops $4.3\%$ above $P^*$). On the remaining cells
ETA is not within $1\%$ at the 600\,s cap ($s=3$--$89\%$).}\label{tab:t1}
\setlength{\tabcolsep}{3.5pt}
\begin{tabular}{llrrrrrl}\toprule
dataset & $C$ & $P^*$ & ETA $1\%$ & ETA $0.1\%$ & cert.\ $1\%$ & cert.\ $10^{-3}$ & LIBLINEAR $1\%$, Newton / dual \\\midrule
gisette & 0.1 & 0.14835 & 17 s & 18 s & 19 s & 19 s & 86 s / 3.0 s \\
gisette & 1 & 0.14893 & 20 s & 21 s & 21 s & 22 s & 1,813 s / 2.9 s \\
gisette & 10 & 0.14899 & 20 s & 21 s & 22 s & 22 s & none at $10^{-6}$ / not run \\
\midrule
a9a & 0.1 & 687.07 & 218 s & 295 s & 402 s & 510 s & 0.5 s / 131 s \\
w8a & 0.1 & 108.80 & 435 s & 559 s & 670 s & 792 s & $\le24$ s / not run \\
ijcnn1 & 0.1 & 550.61 & 80 s & 125 s & 189 s & 259 s & $\le0.14$ s / not run \\
\bottomrule\end{tabular}\end{table}

\begin{table}[h]\centering\scriptsize\setlength{\tabcolsep}{3.5pt}
\caption{Step ablation on real data with a time split, isolated runs
(\texttt{src/step\_profile.py}: three seeds, each configuration alone
in its own process on an otherwise idle AMD EPYC 7B13, with the final
code): the gisette $L_2$ margin ($C{=}1$, $d{=}5{,}000$) and the four
kernel $L_2$ cells of Table~\ref{tab:regime}. MDM is plain pairwise FW
(single transfers, no guard, no fallback) with the prioritised receiver
search and schedule of the other rows; ``exact'' rows replace that
search by an exact linear minimisation over all points; AFW is
away-step FW with exact search; guarded $k{=}1$ is
Algorithm~\ref{alg:block} with a single pair (guard and case-(b)
fallback active); BPCG is the blended pairwise step of \cite{tsuji2022}
on the side supports (Table~\ref{tab:bpcg}). Split of the ETA time: scan = the score scan and
pair selection; col = Gram or kernel column computation; step = the
step itself (block assembly, guard, line search, cache and score
update); rest = certificate and bounds. Drops = away fallbacks that
removed a support index. Every run converged to the certified gap,
distances agreeing to that tolerance and test accuracies to
$2\times10^{-4}$ per dataset and seed, except AFW on the ijcnn1, a9a and
covtype kernel cells, which timed out at $600$\,s with certified gaps of
$0.0012$--$0.004$ (iterations and columns at the cap shown).}\label{tab:kreal}
\begin{tabular}{llrrrrrrrrr}\toprule
instance & variant & $k$ & time (s) & scan & col & step & rest & iterations & columns & drops \\\midrule
gisette $L_2$ & MDM & 1 & $12.1\pm0.2$ & 0.2 & 10.2 & 0.8 & 0.9 & 4,946 & 1,300 & 0 \\
gisette $L_2$ & MDM, exact & 1 & $12.0\pm0.2$ & 0.1 & 10.6 & 0.7 & 0.7 & 3,546 & 1,341 & 0 \\
gisette $L_2$ & guarded & 1 & $12.8\pm0.1$ & 0.2 & 10.5 & 1.5 & 0.6 & 3,471 & 1,334 & 82 \\
gisette $L_2$ & guarded, exact & 1 & $12.9\pm0.4$ & 0.1 & 10.7 & 1.5 & 0.7 & 3,416 & 1,341 & 87 \\
gisette $L_2$ & AFW, exact & 1 & $206.7\pm5.3$ & 7.6 & 10.8 & 90.9 & 97.3 & 527,108 & 1,340 & 0 \\
gisette $L_2$ & block & 16 & $13.0\pm0.1$ & 0.1 & 11.6 & 1.2 & 0.2 & 831 & 1,480 & 301 \\
gisette $L_2$ & BPCG & --- & $12.1\pm0.7$ &  & & & & 7,930 & 1,300 & 75 \\
\midrule
ijcnn1 kernel $L_2$ & MDM & 1 & $7.8\pm0.5$ & 1.2 & 0.8 & 2.2 & 3.6 & 12,763 & 4,523 & 0 \\
ijcnn1 kernel $L_2$ & MDM, exact & 1 & $4.4\pm0.2$ & 0.2 & 0.8 & 1.2 & 2.3 & 5,528 & 4,519 & 0 \\
ijcnn1 kernel $L_2$ & guarded & 1 & $8.1\pm0.4$ & 0.8 & 0.7 & 4.1 & 2.6 & 6,118 & 4,514 & 14 \\
ijcnn1 kernel $L_2$ & guarded, exact & 1 & $7.4\pm0.5$ & 0.1 & 0.7 & 3.9 & 2.6 & 6,070 & 4,520 & 14 \\
ijcnn1 kernel $L_2$ & AFW, exact & 1 & $>600$ & 11.8 & 0.7 & 254.1 & 333.5 & 658,667 & 4,504 & 0 \\
ijcnn1 kernel $L_2$ & block & 16 & $3.8\pm0.3$ & 0.2 & 0.6 & 2.3 & 0.7 & 1,396 & 4,590 & 98 \\
ijcnn1 kernel $L_2$ & BPCG & --- & $4.4\pm0.4$ &  & & & & 12,460 & 4,515 & 21 \\
\midrule
w8a kernel $L_2$ & MDM & 1 & $6.6\pm0.7$ & 0.9 & 1.6 & 1.8 & 2.3 & 11,491 & 3,588 & 0 \\
w8a kernel $L_2$ & MDM, exact & 1 & $3.4\pm0.2$ & 0.1 & 1.4 & 0.8 & 1.1 & 4,307 & 3,585 & 0 \\
w8a kernel $L_2$ & guarded & 1 & $5.2\pm0.5$ & 0.4 & 1.4 & 2.4 & 1.1 & 4,243 & 3,581 & 8 \\
w8a kernel $L_2$ & guarded, exact & 1 & $5.0\pm0.4$ & 0.1 & 1.4 & 2.3 & 1.1 & 4,248 & 3,586 & 11 \\
w8a kernel $L_2$ & AFW, exact & 1 & $225.7\pm3.4$ & 5.3 & 1.5 & 107.1 & 111.7 & 350,050 & 3,569 & 0 \\
w8a kernel $L_2$ & block & 16 & $2.5\pm0.2$ & 0.1 & 1.3 & 0.9 & 0.2 & 603 & 3,591 & 50 \\
w8a kernel $L_2$ & BPCG & --- & $3.4\pm0.2$ &  & & & & 8,909 & 3,581 & 13 \\
\midrule
a9a kernel $L_2$ & MDM & 1 & $23.3\pm0.3$ & 3.1 & 3.1 & 5.8 & 11.3 & 22,879 & 7,609 & 0 \\
a9a kernel $L_2$ & MDM, exact & 1 & $12.8\pm0.1$ & 0.3 & 2.8 & 3.1 & 6.6 & 9,616 & 7,608 & 0 \\
a9a kernel $L_2$ & guarded & 1 & $19.5\pm1.1$ & 1.7 & 2.7 & 8.5 & 6.7 & 9,879 & 7,611 & 27 \\
a9a kernel $L_2$ & guarded, exact & 1 & $18.0\pm0.4$ & 0.2 & 2.6 & 8.3 & 6.8 & 9,852 & 7,607 & 22 \\
a9a kernel $L_2$ & AFW, exact & 1 & $>600$ & 7.4 & 3.0 & 245.6 & 344.2 & 412,530 & 7,543 & 0 \\
a9a kernel $L_2$ & block & 16 & $6.9\pm0.6$ & 0.3 & 2.4 & 3.0 & 1.2 & 1,489 & 7,616 & 35 \\
a9a kernel $L_2$ & BPCG & --- & $10.9\pm0.2$ &  & & & & 18,995 & 7,603 & 24 \\
\midrule
covtype kernel $L_2$ & MDM & 1 & $29.3\pm1.7$ & 3.9 & 2.2 & 7.5 & 15.7 & 26,748 & 9,067 & 0 \\
covtype kernel $L_2$ & MDM, exact & 1 & $16.2\pm0.4$ & 0.3 & 2.0 & 4.3 & 9.6 & 11,789 & 9,066 & 0 \\
covtype kernel $L_2$ & guarded & 1 & $26.2\pm1.8$ & 2.3 & 1.8 & 11.9 & 10.2 & 12,544 & 9,069 & 12 \\
covtype kernel $L_2$ & guarded, exact & 1 & $24.6\pm1.3$ & 0.3 & 1.9 & 11.9 & 10.6 & 12,602 & 9,068 & 14 \\
covtype kernel $L_2$ & AFW, exact & 1 & $>600$ & 6.4 & 2.0 & 243.8 & 348.0 & 354,002 & 9,026 & 0 \\
covtype kernel $L_2$ & block & 16 & $12.2\pm0.2$ & 0.7 & 1.6 & 6.7 & 3.1 & 3,336 & 9,074 & 26 \\
covtype kernel $L_2$ & BPCG & --- & $14.6\pm2.7$ &  & & & & 24,749 & 9,066 & 15 \\
\bottomrule\end{tabular}\end{table}

With the prioritised receiver search the guarded single pair halves
MDM's iterations on ijcnn1 ($1.4\times$ on gisette), taking drop steps
on $0.1$--$2.4\%$ of its iterations; an exact receiver search cuts
iterations as much and, on the kernel cells, time by $1.8$--$1.9\times$,
and with exact search the guard changes iterations by under $10\%$
either way, so that gain is the search's, not the guard's. Blocks then cut
iterations a further $3.8$--$7.0\times$. Away fallbacks that drop an
index occur on $0.1$--$36\%$ of guarded iterations, and their share
grows with $k$: larger blocks add more indices per iteration, and the
drops remove them. The time split says where fewer iterations buy time.
On gisette, $82$--$89\%$ of every configuration's time except
away-step FW's is Gram-column computation ($O(nd)$ per new receiver,
$d{=}5{,}000$), so all but away-step FW lie within $11\%$ of one another at $12$--$13$\,s.
Iteration, column and drop counts equal those of the September runs for
every configuration both contain except plain MDM, whose prioritised
search varies between runs (up to $5\%$ across seeds, under $0.3\%$
between the two runs' means). Peak resident memory is $0.74$\,GB on
gisette and, as per-cell means on the kernel cells, $0.49$--$0.93$\,GB
for ETA and $0.67$--$1.27$\,GB for BPCG. On the kernel cells the columns are cheap and the rest
is per-iteration $O(n)$ work: the score update and line search inside
the step, the certificate in the remainder, and least of all the scan
itself. That work falls with the iteration count, though not in
proportion, because a block step updates the scores with up to $k'$
columns; the net effect on ijcnn1 is $2.1\times$ over MDM, $1.2\times$
over exact-search MDM and $1.2\times$ over BPCG. Away-step FW with exact
search needs $30$--$110\times$ more iterations than plain MDM where it
converges and times out on three kernel cells.

\paragraph{How much blocks help, and choosing $k$.} For uncapped pairs
$S=\sum_jb_j^2$, so $\hat\eta=(D/S-1)/(k'-1)$ is the effective pair
correlation of Proposition~\ref{prop:orth}, and the guard computes it at
no cost. Instrumented seed-0 runs on the five real cells of
Table~\ref{tab:kreal} (\texttt{src/block\_diagnostics.py}, September,
one process at a time, $k\in\{1,4,16,32\}$; $\hat\eta$ is the median over
$k\in\{4,16,32\}$ of each run's median) give $\hat\eta=0.07$
(gisette), $0.11$ (a9a), $0.17$ (w8a), $0.26$ (ijcnn1) and $0.29$
(covtype). The iteration cut of $k{=}16$ over plain MDM is larger at
low correlation ($15\times$ on a9a, $19\times$ on w8a) than at high
($9\times$ on ijcnn1, $8\times$ on covtype), and only $6\times$ on
gisette despite its low correlation, because drop steps take $36\%$ of
its block iterations. By Proposition~\ref{prop:orth}, a 16-pair block of equal pair gains at
these correlations corresponds to $3$--$8$ transfers' worth, against $16$ for orthogonal pairs. On the
synthetic instances the cut grows with $d$ ($6.8\times$ at $d{=}100$,
$14.5\times$ at $d{=}1{,}000$, Table~\ref{tab:kabl}), as expected if
random difference directions become more orthogonal in higher
dimension; we did not measure $\hat\eta$ there. If an iteration costs $a+bk$ (scan plus per-pair
work) and a block is worth $k/(1+\eta(k-1))$ transfers, the cost per
transfer is minimised at $k^*=\sqrt{(a/b)(1-\eta)/\eta}$; with $a$ and
$b$ fitted per cell and $\hat\eta$ for $\eta$ this gives $k^*=7$--$14$,
to be capped by the support size and the column cache. In the September timings the best $k$ was $4$ on gisette and ijcnn1 and
$16$ on the other three cells, each within $1.3\times$ of $k{=}16$; on
gisette, where columns took $93\%$ of the time, no $k$ helped by more
than $1.1\times$.

% generated by src/bpcg_table.py --opt from results/step_profile5.json and results/step_profile_bpcg2.json; do not edit
\begin{table}[htbp]\centering\scriptsize
\caption{BPCG \cite{tsuji2022} versus block ETA ($k{=}16$) on the same objective with the same certificate ($\varepsilon{=}10^{-3}$), the same column cache and the same 600\,s cap ($L_2$ margins at $C{=}1$), three seeds, each configuration alone in a fresh process on the AMD EPYC 7B13 (\texttt{src/step\_profile.py}). ``product'': BPCG on the $nm$ product atoms $v_i-w_j$ (\texttt{src/bpcg.py}); ``side'': the same blended pairwise rule on the side supports ETA uses (\texttt{src/corrective\_fw.py}). Both BPCG rows were run after a fix to their kernel-column routine, which had recomputed a class's squared norms for every RBF column ($3.5\times$ ETA's column cost; columns are bit-identical). Every run converged; distances agree per cell and seed to $10^{-6}$ and test accuracies to $2\times10^{-4}$. Columns are kernel or Gram columns computed; drops count ETA's away drops and BPCG's capped local steps.}\label{tab:bpcg}
\begin{tabular}{llrrrr}\toprule
cell & solver & time (s) & iterations & columns & drops \\\midrule
gisette $L_2$ & ETA, $k{=}16$ & $13.0\pm0.1$ & 831 & 1,480 & 301 \\
gisette $L_2$ & BPCG, product & $13.9\pm0.8$ & 10,470 & 1,322 & 87 \\
gisette $L_2$ & BPCG, side & $12.1\pm0.7$ & 7,930 & 1,300 & 75 \\
\midrule
gisette KL2 & ETA, $k{=}16$ & $35.1\pm6.0$ & 251 & 3,653 & 4 \\
gisette KL2 & BPCG, product & $39.7\pm3.4$ & 8,702 & 3,641 & 3 \\
gisette KL2 & BPCG, side & $36.2\pm4.0$ & 4,375 & 3,643 & 3 \\
\midrule
w8a KL2 & ETA, $k{=}16$ & $2.5\pm0.2$ & 603 & 3,591 & 50 \\
w8a KL2 & BPCG, product & $10.9\pm1.3$ & 12,342 & 3,577 & 12 \\
w8a KL2 & BPCG, side & $3.4\pm0.2$ & 8,909 & 3,581 & 13 \\
\midrule
a9a KL2 & ETA, $k{=}16$ & $6.9\pm0.6$ & 1,489 & 7,616 & 35 \\
a9a KL2 & BPCG, product & $46.6\pm1.3$ & 28,653 & 7,605 & 24 \\
a9a KL2 & BPCG, side & $10.9\pm0.2$ & 18,995 & 7,603 & 24 \\
\midrule
covtype KL2 & ETA, $k{=}16$ & $12.2\pm0.2$ & 3,336 & 9,074 & 26 \\
covtype KL2 & BPCG, product & $69.3\pm3.6$ & 38,422 & 9,065 & 15 \\
covtype KL2 & BPCG, side & $14.6\pm2.7$ & 24,749 & 9,066 & 15 \\
\midrule
ijcnn1 KL2 & ETA, $k{=}16$ & $3.8\pm0.3$ & 1,396 & 4,590 & 98 \\
ijcnn1 KL2 & BPCG, product & $16.2\pm1.9$ & 18,034 & 4,512 & 19 \\
ijcnn1 KL2 & BPCG, side & $4.4\pm0.4$ & 12,460 & 4,515 & 21 \\
\bottomrule
\end{tabular}\end{table}

Both BPCGs use the separable linear minimisation oracle of
Lemma~\ref{lem:sep}, exact line searches, the certificate of
Lemma~\ref{lem:gapid} and the same cached columns, so column counts are
comparable. On the product atoms BPCG is $4.3$--$6.8\times$ slower than
block ETA on the four non-gisette kernel cells, but most of that is the
parametrisation: on the side supports ETA uses, the same rule needs
$7$--$17\times$ more iterations than block ETA and is $1.2$--$1.6\times$
slower there (separated intervals on w8a and a9a only), ties on gisette's
kernel margin and is $8\%$ faster on its $L_2$ margin, where columns
dominate. The block's fewer iterations buy time where per-iteration work
dominates and nothing where columns do.

\begin{table}[h]\centering\scriptsize
\caption{Block-size ablation on the synthetic instances (released code, September run): single MDM ($k{=}1$) versus blocks of $k$ transfers. Iterations fall with $k$; wall-clock need not:
blocks touch more columns, and on the $d{=}100$ $L_2$ instance they add
receivers ($1{,}345$ to $3{,}076$ columns) and single MDM is fastest.}\label{tab:kabl}
\begin{tabular}{lrrrrr}\toprule
instance & $k$ & time (s) & iterations & columns & drops \\\midrule
hard margin, $d{=}1000$, $\varepsilon{=}10^{-5}$ & 1 & 0.7 & 301 & 104 & 0 \\
hard margin, $d{=}1000$, $\varepsilon{=}10^{-5}$ & 4 & 0.7 & 36 & 103 & 0 \\
hard margin, $d{=}1000$, $\varepsilon{=}10^{-5}$ & 16 & 0.7 & 16 & 103 & 0 \\
hard margin, $d{=}1000$, $\varepsilon{=}10^{-5}$ & 32 & 0.7 & 16 & 109 & 0 \\
\midrule
$L_2$ overlap, $d{=}100$, $C{=}1$ & 1 & 17.4 & 72,546 & 1,345 & 0 \\
$L_2$ overlap, $d{=}100$, $C{=}1$ & 4 & 23.1 & 24,606 & 1,632 & 416 \\
$L_2$ overlap, $d{=}100$, $C{=}1$ & 16 & 21.7 & 10,701 & 2,404 & 599 \\
$L_2$ overlap, $d{=}100$, $C{=}1$ & 32 & 25.8 & 7,606 & 3,076 & 305 \\
\midrule
$L_2$ overlap, $d{=}1000$, $C{=}1$ & 1 & 23.4 & 43,486 & 1,242 & 0 \\
$L_2$ overlap, $d{=}1000$, $C{=}1$ & 4 & 18.5 & 10,071 & 1,330 & 177 \\
$L_2$ overlap, $d{=}1000$, $C{=}1$ & 16 & 15.1 & 3,001 & 1,348 & 180 \\
$L_2$ overlap, $d{=}1000$, $C{=}1$ & 32 & 15.1 & 1,741 & 1,396 & 170 \\
\bottomrule
\end{tabular}\end{table}

\begin{figure}[h]\centering
\includegraphics[width=.7\linewidth]{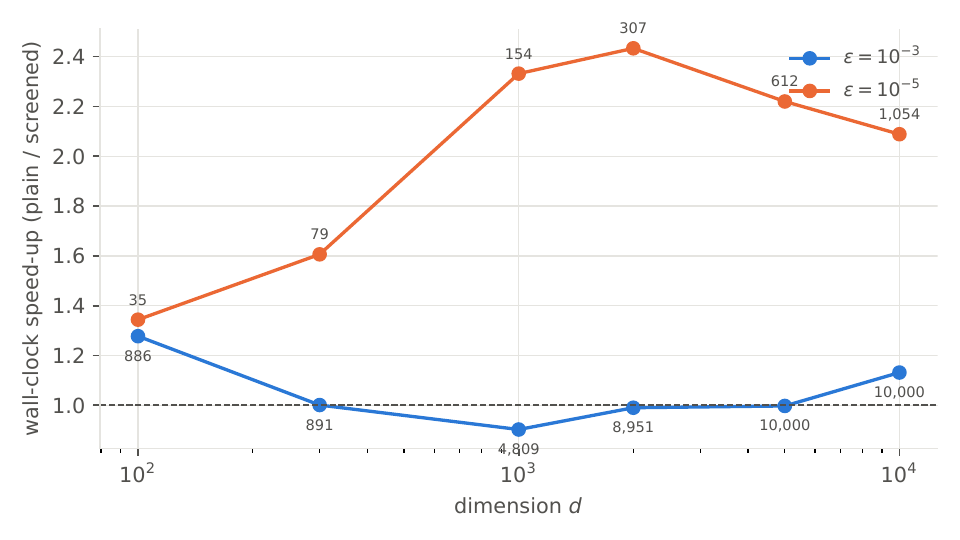}
\caption{Safe screening: wall-clock speed-up vs dimension at two
tolerances on the separated synthetic protocol; labels are surviving
points of $10{,}000$. Made with the toward-step solver of the
development phase (\texttt{src/shrink\_benchmark.py}); the rule and its
activation do not depend on the step.}\label{fig:shrink}
\end{figure}

\begin{figure}[h]\centering
\includegraphics[width=.78\linewidth]{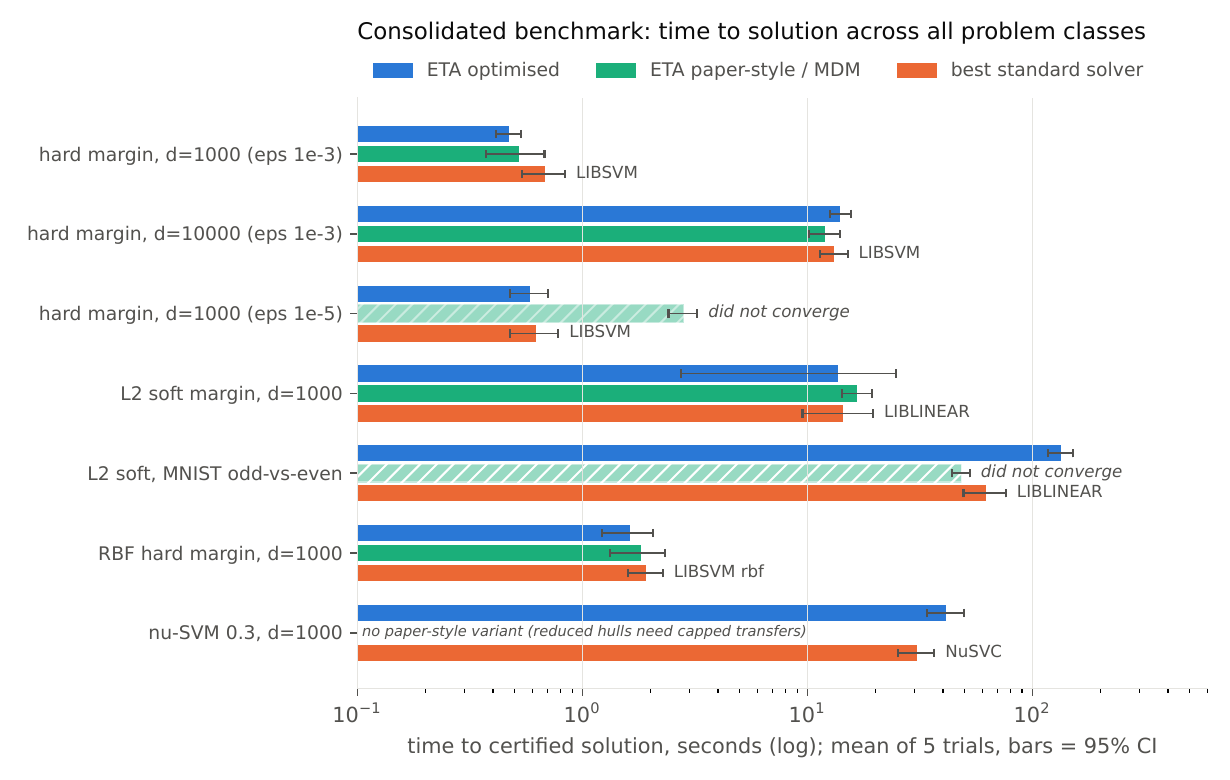}
\caption{Time to certified solution on seven problem classes: optimised
ETA ($k{=}16$ with screening on the hard-margin classes, $k{=}32$ on the
$L_2$ class, $k{=}8$ on MNIST) vs the comparison configuration (the
toward-step algorithm of \cite{gupta2026} on the hard-margin classes,
single MDM on the $L_2$ and MNIST classes, pairwise K-ETA on RBF) vs the
best standard solver per class; five trials per class (October rerun, AMD EPYC 9754), mean with
95\% CI. The $\nu$-SVM row has no comparison bar because the toward-step
algorithm of \cite{gupta2026} has no reduced-hull variant (reduced
hulls need capped transfers).}\label{fig:final}
\end{figure}

\section{Extended related work}\label{app:related}
\paragraph{Aggregated directions.} Boosted FW \cite{combettes2020}
builds its direction by greedily aligning a combination of atoms with
the negative gradient, at several oracle calls per iteration, and
inherits the FW rate through an alignment condition; our block fixes the
pair coefficients $\gamma_j$ from one scan and relies on the guard, not
on alignment, for its rate.
\paragraph{Swap steps.} Blended pairwise conditional gradients
\cite{tsuji2022} also remove pairwise FW's swap-step bound with a guard,
there between the pairwise step and a FW step, and \cite{rinaldi2023}
avoid bad steps with a short-step chain. Our guard compares an
aggregated block of $k$ transfers with the single MDM step and falls
back to an away step, so the drop-step count of away-step FW applies
unchanged. Block-coordinate FW \cite{lacoste2013bcfw} blocks over
\emph{coordinates} of a product domain rather than aggregating pairs on
one polytope.
\paragraph{Screening and identification.} Gap-safe rules
\cite{fercoq2015,ndiaye2017,shibagaki2016} screen with a radius from
the duality gap and identify the support in finite time; the active-set
complexity of away-step FW \cite{bomze2020} and the
strict-complementarity bounds of \cite{garber2020} give the same
identification from the algorithm's side. Theorem~\ref{thm:activation}
is the polytope-distance instance, with constants tied to the TA
certificate; the SVM rules of \cite{elghaoui2012,ogawa2013} and the
exact data reduction of \cite{wang2014} are its ancestors.
\paragraph{Relation to other FW variants.} The toward, away and pairwise counts of the synthetic step ablation and
the MDM, AFW and block rows of Table~\ref{tab:kreal} are Frank--Wolfe,
away-step FW, pairwise FW and our block step on the same
polytope and objective; blended pairwise conditional gradients (BPCG)
\cite{tsuji2022} is compared on real data in Table~\ref{tab:bpcg};
Wolfe's minimum-norm-point method \cite{wolfe1976} was not run. BPCG's
blending rule compares the away
and FW gaps and its linear rate has the same pyramidal-width form as
Theorem~\ref{thm:linear}; our case (b) instead drops when the away step is clipped and
otherwise compares the realised gains of the block, MDM, away and
toward candidates, so each non-drop step is at least as good as the
best single candidate, and it is the away
fallback rather than the blending that yields the drop-step count. The
facial distance of \cite{pena2019} is a computable alternative to the
pyramidal width; a bound on $\delta$ for the difference atom set
$A=\{v_i-w_j\}$ in terms of the two class hulls is open.
\paragraph{Synthetic results in brief.} Away and pairwise steps
converge at $\varepsilon=10^{-6}$ where the toward and midpoint steps of
\cite{gupta2026} do not. Screening reduces $10{,}000$ points to
$36$--$1{,}055$ survivors at $\varepsilon=10^{-5}$. On seven synthetic
and small-real problem classes (Figure~\ref{fig:final}, five trials
each) the optimised solver is between $1.5\times$ faster and $1.1\times$
slower than LIBSVM on the four hard-margin classes, ties LIBLINEAR on
the synthetic $L_2$ margin, and is $1.4$--$2.1\times$ slower than NuSVC
and LIBLINEAR on the $\nu$-SVM and MNIST classes. On the four classes
where the comparison configuration also converges, blocks and guard
change wall-clock by $0.85$--$1.2\times$; their clear wins are the two
where it does not (the toward step at $\varepsilon=10^{-5}$, single MDM
on MNIST).

\section{Numerical verification of the analysis}\label{app:num}
Lemma~\ref{lem:block} was checked on $4{,}000$ random ensembles
including heavily clipped and boundary cases (worst slack at machine
precision); Proposition~\ref{prop:orth} on $500$ near-orthogonal
ensembles; Lemma~\ref{lem:gapid} to $10^{-9}$ at arbitrary iterates;
Lemma~\ref{lem:twosided} on instrumented two-sided runs (no drop
iteration added a support index, and every non-drop iteration met the
gain bound with $M$ replaced by the computable upper bound
$\operatorname{diam}V+\operatorname{diam}W$); and Theorem~\ref{thm:activation} on instances whose optimal
face was computed by an exact QP: at every activating round all
zero-weight non-face points were screened, and the final working set
equalled the optimal face exactly. Over $5{,}000$ instrumented live
block steps none increased the objective. All checks are test files in
the released code.

\end{document}